\documentclass{llncs}
\usepackage[latin1]{inputenc}
\usepackage{array}
\usepackage{amsmath}
\usepackage{amsfonts}
\usepackage{amssymb}
\usepackage{amsbsy}
\usepackage{graphicx}
\usepackage{epsfig}
\usepackage{verbatim}
\usepackage{array}
\usepackage{color}
\usepackage{algorithm}
\usepackage{algorithmic}
\usepackage{multirow}
\usepackage{wrapfig}
\usepackage{subfigure}
\usepackage{url}
\usepackage{floatflt}
\usepackage{wrapfig}
\usepackage{times}

\newenvironment{coden}
	{\begin{tt}\begin{tabbing}12345\=12\=12\=12\=12\=12\=\kill}
	{\end{tabbing}\end{tt}}

\newcommand{\caplab}[2]{\caption{\label{#1} #2}}
\newcommand{\tabbegin}[1]{\begin{table*}[#1]\centering}
\newcommand{\tabend}{\end{table*}}
\newcommand{\figbegin}[1]{\begin{figure*}[#1]\centering}
\newcommand{\figend}[2]{\caplab{#1}{#2}\end{figure*}}

\begin{document}

\title{A Combinatorial Optimisation Approach to
 Designing Dual-Parented Long-Reach Passive Optical Networks\thanks{This paper was supported by Science Foundation Ireland under grant 08/CE/I1423.
}}
%

\author{Hadrien Cambazard, Deepak Mehta, Barry O'Sullivan, and Luis Quesada\inst{1} \\[0.25em]
Marco Ruffini, David Payne, and Linda Doyle\inst{2}}
\institute{CTVR \& Cork Constraint Computation Centre,
 University College Cork, Ireland
 \email{\{h.cambazard|d.mehta|b.osullivan|l.quesada\}@4c.ucc.ie} \\[0.5em]
\and  CTVR, University of Dublin, Trinity College, Ireland
\email{ruffinm@tcd.ie|david.b.payne@btinternet.com|ledoyle@tcd.ie}
}
\maketitle

\begin{abstract}
We present an application focused on the design of
  resilient long-reach passive optical networks.
We specifically consider dual-parented networks whereby each customer
  must be connected to two metro sites via local exchange sites.
An important property of such a placement is resilience to
  single metro node failure.
The objective of the application is to determine the optimal position
  of a set of metro nodes such that the total optical fibre length
  is minimized.
We prove that this problem is NP-Complete.
We present two alternative combinatorial optimisation
  approaches to finding an optimal metro node placement using:
  a mixed integer linear programming (MIP) formulation of the problem;
  and, a hybrid approach that uses clustering as a preprocessing step.
 We consider a detailed case-study based on a network for Ireland.
The hybrid approach scales well and finds solutions that are close
  to optimal, with a runtime that is two orders-of-magnitude better than
  the MIP model.

 \end{abstract}

\section{Introduction}
Over the past decade telecommunications network traffic has grown exponentially
 at an average annual rate above $75\%$ prompted 
by a multitude of new on-line content sharing applications such as Facebook and YouTube. Although the forecast for traffic growth 
over the next 5 years 
is reduced, it  still suggests an average annual compound rate of about $37\%$ with Internet video applications growing at 
about $47\%$.
As High Definition (HD) and 3D video will increasingly be delivered over the Internet, such forecasts do not seem to over-estimate
the traffic scenario.  Additionally, delivering high peak data rates becomes increasingly important for  delivering satisfactory quality of experience, especially for real-time services.   Fiber-To-The-Premises (FTTP), and
 in particular Fiber-To-The-Home (FTTH), seems to be the only solution capable of providing scalable access bandwidth for the foreseeable future.

Passive Optical Networks (PONs) are widely recognized as an economically viable solution to deploy FTTP and FTTH, by virtue of the ability to share costly equipment and fibre among a number of customers. In particular,  the Long-Reach PON (LR-PON) is gaining interest.
LR-PON provides  an economically viable solution as
the number of active network nodes can be reduced by  two orders-of-magnitude and all electronic data processing can be removed from the local exchange sites, thereby reducing both cost and energy consumption \cite{PayneECOC09}. 
However, a major fault occurrence like a complete failure of  a single metro node
 that terminates the LR-PON could affect tens of thousands of customers.
Therefore, protection against a metro node failure is of primary importance
 for the LR-PON-based architecture.

\begin{figure}[t]
\centering
\includegraphics[width=0.75\textwidth]{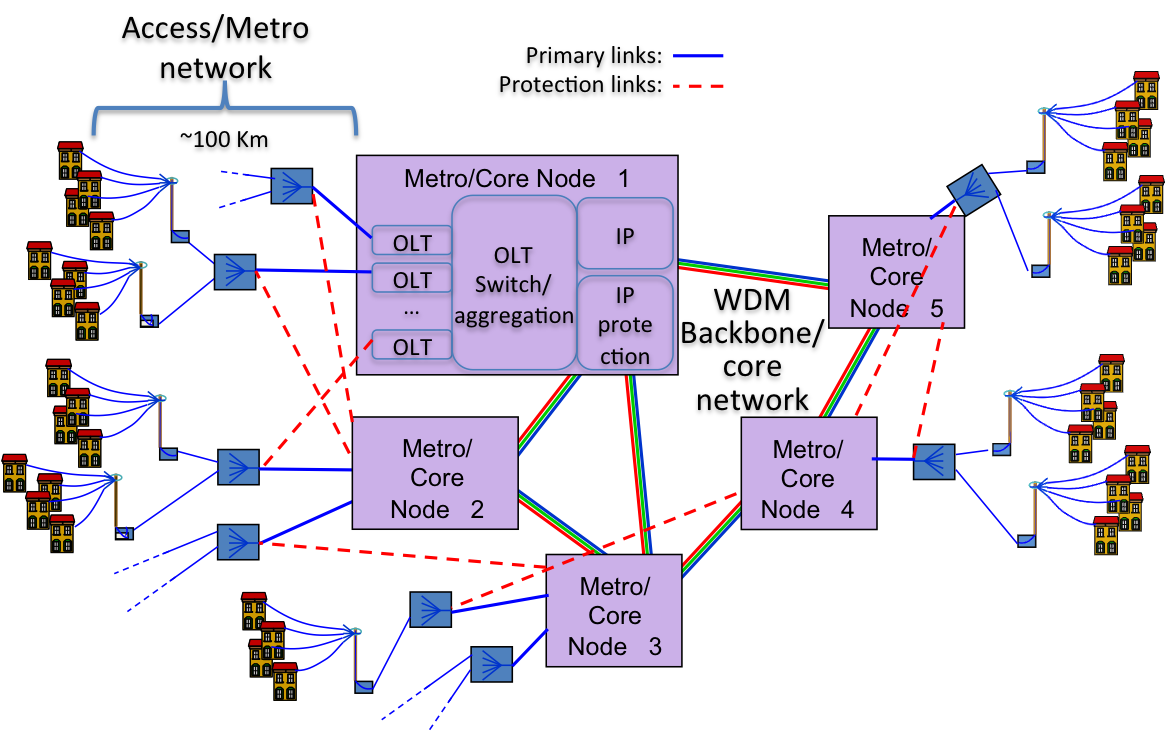}
\caption{Example of a LR-PON showing both access/metro and backbone/core networks}
\label{fig:PON-network}
\end{figure}

A basic and effective protection mechanism for LR-PON is to dual parent each system onto two metro/outer core nodes \cite{PON-amplified-protection,HunterJOptNetw2007}. 
This is similar to a simple protection solution
 for IP routers known as double or redundant protection~\cite{DeMaesschalckComMag2002}.
Figure \ref{fig:PON-network} shows an example of a PON network, together with its Wavelength-division Multiplexing (WDM) backbone interconnections. Each PON is dual-parented, with the dashed lines representing the protection links. In this work we have considered protection links up to the first PON split (or local exchange site), leaving the ``last mile'' unprotected. This is a common choice for residential customers, while protection can be extended to the user premises for business customers.
For example, considering Figure~\ref{fig:PON-network}, if metro node 1 fails, its PONs will be protected by
metro node 2. Of course, node 2 needs to be over-provisioned with much larger IP capacity
in order to protect the additional load~\cite{RuffiniECOC10}.
Providing such a protection mechanism can significantly increase network overall cost
because fibre deployment is a significant contributor to the total cost of the PON installation.  
Therefore, we focus on the problem of finding an optimal set of  positions  of $k$ metro nodes 
such that the cost of connecting optical fibres between  metro nodes and  exchange sites
is minimized. 
The set of possible positions
available for a metro node
is the set of positions associated with  the existing old local exchange sites.

\section{Problem Formalization}
\label{sec:prbformal}
We formally describe the problem
of LR-PON deployment for a real geography based on the  data provided by the Irish incumbent operator. 
More precisely, 
we present the definition and the complexity of the so called {\it Single Coverage Problem} where each exchange site is only connected to a single metro node
and then present the definition and complexity of the {\it Double Coverage Problem} where each exchange site is connected to two metro nodes. 

\begin{definition}[Single Coverage Problem]
An instance of the Single Coverage Problem (SCP)  is defined by $\langle A,B,c,k,\phi \rangle$, where
  $(A,B)$ is a complete bipartite graph with cost function $c$ such that  $c_{ij}$ is the cost of allocating node $i \in A$ to node $j \in B$,  
  $k$ is an integer value such that $k \leq |B|$ and $\phi$ is some real value.
An allocation from $A$ to   $M \subseteq B$ maps each node $i$ of A to the cheapest node $j$ of $M$ such that $j = \arg\min_{j \in M} c_{ij}$
and $k=|M|$. The total cost of the allocation is the sum of the allocation cost of each node of $A$. The problem is to verify whether there exists a subset of $k$  nodes of $B$ such that the total cost is less than or  equal to $\phi$.

\end{definition}

\begin{proposition}
The Single Coverage Problem is  NP-Complete. 
\end{proposition}
\begin{proof}
A reduction from \textsc{Hitting Set Problem}, which is known to be NP-complete \cite{Garey:1979}, is obtained as follows: given a collection $C$ of subsets of a finite set $S$ and a positive integer $m \leq |S|$, the \textsc{Hitting Set} problem is to decide whether there is a subset $S' \subseteq S$ with $|S'| \leq m$ such that $S'$ contains at least one element from each subset in $C$. The reduction to SCP, $\langle A,B,c,k,\phi \rangle$, goes as follows. We have a node in $A$ for each set $S_i$ in $C$ and a node in $B$ for each $j \in S$ . The cost of all edges $(i,j)$ is either 0 if $j$ is in $S_i$ or $1$ otherwise. We set $\phi=0$ and $k=m$. The constructed instance of SCP has a solution of cost 0 if and only if there exists a hitting set of size $m$ for $C$.
$\hfill\blacksquare$
\end{proof}

\begin{definition}[Double Coverage Problem]
An instance of the Double Coverage Problem (DCP) is also  defined by $\langle A,B,c,k,\phi \rangle$, where
  $(A,B)$ is a complete bipartite graph with cost function $c$, where  $c_{ij}$ is the cost of allocating node $i \in A$ to node $j \in B$,  
  $k$ is an integer value such that $k \leq |B|$ and $\phi$ is some real value.
An allocation from $A$ to   $M \subseteq B$ maps each node $i$ of $A$ to the cheapest node $j_1$ of $M$ such that $j_1 = \arg\min_{j_1 \in M} c_{ij_1}$, and to the second cheapest node $j_2$ of $M$ such that $j_2 = \arg\min_{j_2 \in M|j_2 \neq j_1} c_{ij_2}$. The total cost of the allocation is the sum of the allocation costs of each node of $A$ to \textbf{two} nodes of $B$. The problem is to verify whether there exists a subset of $k$ nodes of $B$ such that the total cost is less than or  equal to $\phi$.
\end{definition}

\begin{proposition}
The Double Coverage Problem is NP-Complete. 
\end{proposition}
\begin{proof}
We can reduce \textsc{SCP}, which was proved to be NP-complete,  to \textsc{DCP} by adding one extra node  to $B$ and setting the cost function accordingly. More precisely, let $B'=B \cup \{s\}$. Let $c'$ be the cost function such that $c'_{ij}= c_{ij}$ if $i \in A$ and $j \in B$,
otherwise $c'_{ij} = \beta$ such that $\beta < \min_{i \in A,j \in B} c_{ij}$. Solving the \textsc{SCP} instance $\langle A,B,c,k,\phi\rangle$ is equivalent to solving the \textsc{DCP} instance $\langle A,B',c',k+1,\phi+|A|\times \beta\rangle$.
Notice that any solution of the \textsc{SCP} instance can be transformed into a  \textsc{DCP} solution by setting $s$ as the cheapest node for every node in  $A$ and making the \textsc{SCP} allocation equivalent to the allocation of the second cheapest node in the  \textsc{DCP} instance. Similarly, any solution of the \textsc{DCP} instance can be transformed into a  \textsc{SCP} solution by ignoring the cheapest node since the cost associated with the allocation of the cheapest nodes is bound to be equal to or greater than $|A|\times \beta$, making the  allocation of the second cheapest node a valid solution of the \textsc{SCP} instance.
$\hfill\blacksquare$
\end{proof}

In this paper we  focus on the double coverage problem where both  $A$ and $B$ are sets of exchange sites.
Let $E$ be a set of exchange sites whose locations are fixed.
In Figure~\ref{fig:Ireland}  all the points are locations of exchange
 sites in Ireland.\footnote{Notice that some points are outside the boundary of Ireland. This is because of the projection of the map of Ireland we are using in this figure.}
Let $l_i$ be the load of the exchange site $i \in E$ which is equivalent to the number of customers that are connected to the
exchange site $i$.
 \begin{floatingfigure}[r]{6.0 cm}
 \scalebox{1.00}{
      \includegraphics[width=5.3cm]{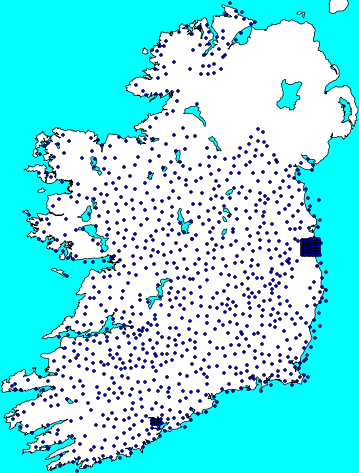} }
\caption{Exchange Sites in Ireland}
\label{fig:Ireland}
 \end{floatingfigure}
Let $k$ be the number of metro nodes that are required to be placed in Ireland. 
A metro node can be placed at any position where an exchange site is located. 
Thus, the set of positions available for each metro node is the set of positions of all the exchange sites. 
Let $d$ be a matrix where $d_{ij}$ denotes the Euclidean distance between the positions of exchange sites  $i$ and  $j$.
In order to account for the fact that the amount of fibre needed to connect two network points is usually larger than their Euclidean distance, because fibre paths generally follow the layout of the road network, a routing factor of $1.6$ is applied.
%
Let $c_{ij}$ be the cost of connecting exchange site $i$ to a metro node placed  at the location of an exchange site $j \in E$, 
which is computed as follows:
\[ c_{ij} = 1.6 \times d_{ij} \times \alpha_i \times l_i. \]
This cost model is  based on the work of one of the authors while working at
 BT~\cite{PayneECOC09}.
Here $\alpha_i$ is constant and its  value  is dependent on the load of the exchange site $i$.
The value of $\alpha_i$ decreases as the load increases  
since sharing of the fibre increases.
The aim is to determine the positions of $k$ metro nodes such that each exchange site is connected to two metro nodes and
the sum of the costs of the connections between exchange sites and their respective metro nodes is minimized. 

%

\section{MIP Model}
\label{sec:mip}
 The objective is to place a number of metro nodes  such that  the cost  of the connection  between the local exchanges and their corresponding  metro nodes is minimized. The closest metro node of an exchange site is called the primary metro node while the second closest is called 
the secondary metro node.

\paragraph{Constants.}
Let $E$ be a set of exchange sites whose locations are fixed. 
Let $k$ be the number of metro nodes whose positions are to be determined.
Let $c_{ij}$ be the cost of connecting an exchange site $i$ to a metro node placed at the location of an exchange site  $j \in E$.

\paragraph{Variables.}

$\forall (i,j) \in E \times E, x_{ij} \in \{0,1\}$ denotes whether exchange site  $i$ is connected to a metro node $j$. 
$\forall j \in E, y_{j} \in \{0,1\}$ denotes whether  $j$ is used as a metro node.

\paragraph{Constraints.}

Each exchange site $i \in E$ should be connected to two metro nodes:
\begin{equation}
\label{con:m2}
 \forall i \in E: \sum_{j \in E} x_{ij} = 2. 
\end{equation}

Constraint ({\ref{con:m2}) implicitly enforces that the primary and secondary metro nodes of exchange site $i$ should be different. 
For each exchange site $i \in E$ its primary and secondary metro nodes can be inferred based on the 
costs of connecting $i$ to the metro nodes respectively. 
If the  metro node connected to an exchange site $i \in E$ is placed at the location of exchange site $j$ then
$y_j$ is one:
   \begin{equation}
\label{con:dep}
 \forall i,j \in E:  y_j \geq x_{ij}. 
   \end{equation}
The number of used locations for metro nodes should be  equal to $k$:
\[  \sum_{j \leq n} y_j = k. \]
The number of constraints of type (\ref{con:dep}) is ($n^2$) which can grow quickly  for large values of $n=|E|$, in which case they can be replaced by the following  weaker constraint:
\[ \forall j \in E: |E| \times y_j \geq \sum_{i \leq n} x_{ij}. \]

\paragraph{Objective.}

The  objective is to minimize the cost of the connection  between local exchanges and their corresponding metro nodes, i.e, 
\[  \min  \sum c_{ij} \times x_{ij}.  \]

\section{Cluster-Based Sampling}
\label{sec:hybrid}
For the {\scshape MIP} model, as presented in the previous section,
the  set of positions of all exchange sites is considered as
the domain of the  metro node position  for each exchange site. 
This may  prohibit us from solving the problem optimally as the size of the set of the positions of the exchange sites increases. 
%
In order to overcome this scalability issue both in terms of time and space, we propose a heuristic approach as a preprocessing step for selecting a small subset of metro node positions for each exchange site and then use the {\scshape mip} model to solve the problem optimally.

One simple approach to overcome this  could be to limit the number of metro node positions  of each exchange site 
based on their distances from the exchange site. More precisely, select $k$ closest/cheapest metro node positions 
for each exchange site  $e$.  
This heuristic approach is called {\it $k$-cheapest neighbours ({\scshape KCN})}.
One of the drawbacks of this approach is that the resulting problem can be inconsistent especially when $k$ is small. 
Therefore, it is important to find a  value of $k$ such that the problem is  satisfiable. 
  Another issue is that  an optimal  solution of the resulting problem may not be of good quality
  despite the problem being satisfiable depending on the value of $k$.
Obviously when $k= |E|$ we will always find the best  solution but at the expense of more time.
 There is a trade-off between the value of $k$ and the time required to find a good solution.

 \begin{floatingfigure}[r]{7.0 cm}
 \scalebox{1.00}{
      \includegraphics[width=6.0cm]{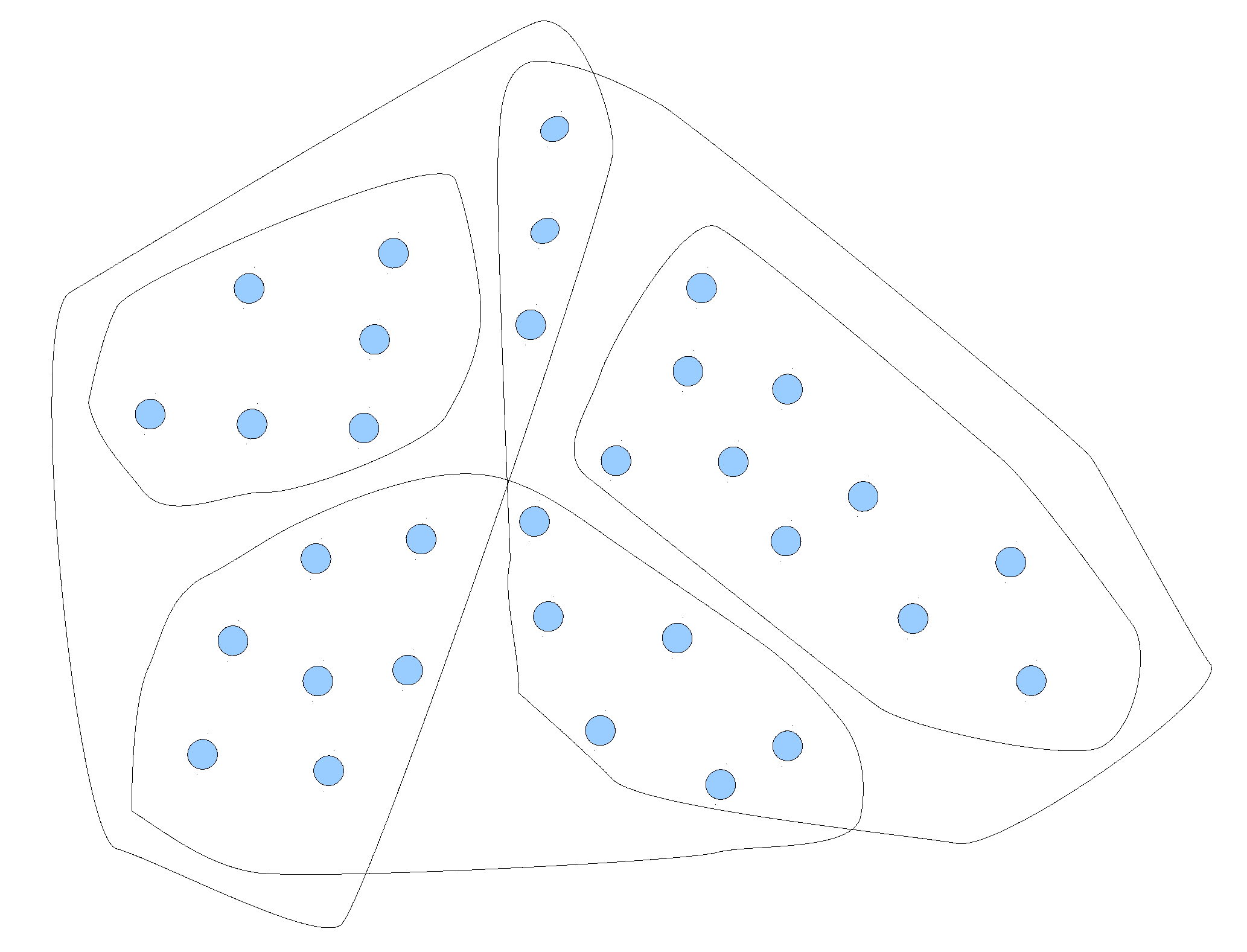} }
\caption{Five overlapping clusters. \label{fig:clusters}}
 \end{floatingfigure}

We  propose a new approach for computing a sample  of  positions where  a metro node can be placed at a given exchange site. 
This heuristic approach is  called {\em cluster-based sampling ({\scshape CBS})}.
The pseudo-code is depicted in Algorithms~\ref{alg:cluster} and \ref{alg:sample}.
The general idea is to apply a variant of the  k-means  algorithm \cite{Lloyd82leastsquares} for computing $k$ clusters of  exchange sites.
Whenever a local minimum is reached within the algorithm, a best exchange site position,
 based on some criterion, is selected from each cluster as a possible
  location of the metro node for all the exchange sites within that cluster. 
 A  sample of positions for each exchange site is computed  by  repeating this process  a given number of times.  
The cardinality of  this set  is considerably smaller than the full set of positions of exchange sites. 
Since each exchange site should be connected to two metro nodes the algorithm for weighted k-means clustering
is adapted
 to ensure that each exchange site is in exactly two clusters.

  \begin{algorithm}[t]
  \footnotesize \caption{\textsc{computeOverlappingClusters}( $E$, $k$ )
  \label{alg:cluster}}
  \begin{scriptsize}
  \begin{coden}
      cost  $\leftarrow \infty$ \\
      select $k$ points randomly from $E$ and assigned them to    $m_1, m_2, \ldots, m_k$ \\
      loop $\leftarrow$ {\bf True} \\ 
     {\bf While} loop  {\bf do}\\
     \>  $\forall i \leq k:$ $P_i \leftarrow \{x_j| \forall i* \leq k:$   dist$(x_j,m_i) \leq$ dist$(x_j,m_{i*})     \} $\\
       \>$\forall i \leq k:$ $S_i \leftarrow \{x_j| \exists i* \leq k:  $ dist$(x_j,m_{i*}) \leq$ dist$(x_j,m_{i}) \:\: \wedge$ \\
      \> \> \> \> \> \> $\quad \quad \quad \forall l \neq i* : $ dist$(x_j,m_i) \leq$ dist$(x_j,m_{l}) \} $\\

      \> newcost $\leftarrow \sum_{1 \leq i \leq k} \sum_{x_j \in P_i \cup S_i}  W[x_j] \times$dist$(x_j,m_i)$\\
      \> {\bf If} cost $>$ newcost \\
      \> \> cost $\leftarrow$ newcost \\
      \> \> $P* \leftarrow P$, $S* \leftarrow S$ \\
    \> \>{\bf For} $i = 1,\ldots, k$ {\bf do} \\
 \> \> \>    $m_i.X \leftarrow \sum_{x_j \in P_i \cup S_i } W[x_j]\times x_j.X/ \sum_{x_j \in P_i \cup S_i } W[x_j]$  \\
    \> \> \> $m_i.Y \leftarrow \sum_{x_j \in P_i \cup S_i } W[x_j]\times x_j.Y/ \sum_{x_j \in P_i \cup S_i} W[x_j]$  \\
      \> {\bf Else} \\
      \> \> loop $\leftarrow$ {\bf False} \\ 
    {\bf Return} $\langle P*,S* \rangle$
  \end{coden}
 \end{scriptsize}
  \end{algorithm}
 
 Algorithm~\ref{alg:cluster}  computes $k$ overlapping clusters. 
 An example is presented in Figure~\ref{fig:clusters} where the value of $k$ is 5. 
Notice that each point is  present in two clusters.
The algorithm \textsc{computeOverlappingClusters} starts by  selecting $k$ points, $\langle m_1,\ldots,m_k\rangle$, randomly from a given set of sites $E$. These points represent initial $k$ means of the overlapping clusters.
Each $m_i$ is associated with two attributes: $m_i.X$ denotes the $X$ dimension and $m_i.Y$ denotes the $Y$ dimension. 
Initially the cost is set to infinity. 
Each exchange site is assigned to two clusters: the one associated  with the closest mean and another with the second closest mean.
In the algorithm a cluster $i$ is represented by $P_i \cup S_i$
such that $m_i$ is the closest
mean  for each $p \in P_i$  and $m_i$ is the second closest mean for each $p \in S_i$.
We use \texttt{dist$(p_i,p_j)$} to denote the Euclidean distance between the points $p_i$ and $p_j$ and
\texttt{W$[p_i]$} to denote the weight associated with a site $p_i$, which is equivalent to $\alpha_i \times l_i$ for our problem.
The cost is evaluated by  summing  the weighted distances of all the points of the clusters with respect to their corresponding means.
If the new cost is less than the current cost then the new means are calculated for all clusters.
 While the new cost is better than the previous cost 
 the assignment of the exchange sites to two clusters and the update of the means is repeated.
The algorithm returns the tuple  $\langle P*,S* \rangle$.
The complexity of each iteration within the while loop of Algorithm~\ref{alg:cluster} is $\mathcal{O}(n\,k)$, where $n$ is the number of sites and $k$ 
is the required number
of metro nodes (or the number of clusters).


\begin{algorithm}[t]
 \footnotesize \caption{\textsc{SamplingPoints}( nbruns, $E$, $k$ )
\label{alg:sample}}
 \begin{scriptsize}
 \begin{coden}
   crun $\leftarrow$ 0 \\
   $\forall x_j \in E$, Pos$(x_j) \leftarrow \emptyset$ \\
   {\bf While} crun $<$ nbruns  {\bf do} \\
     \>  crun $\leftarrow$ crun $+ 1$     \\
     \> $\langle P,S \rangle  \leftarrow$  computeOverlappingClusters$(E, k)$\\
   \> {\bf For} $i = 1,\ldots, k$ {\bf do} \\ 
   \> \> {\bf If} $P_i \neq \emptyset$  {\bf then}\\ 
   \> \> \> select  $s \in P_i$ such that \\
   \> \> \> $\forall s' \in P_i \,\, \sum_{x_j \in P_i \cup S_i} W[x_j] \times $dist$(s,x_j) \leq \sum_{x_j \in P_i \cup S_i}W[x_j] \times$dist$(s',x_j)$  \\
   \> \> {\bf Else if} $S_i \neq \emptyset$ \\
	\> \> \> select  $s \in S_i$ such that \\
    \> \> \>  $\forall s' \in S_i \,\, \sum_{x_j \in S_i}W[x_j] \times$dist$(s,x_j) \leq \sum_{x_j \in  S_i}W[x_j] \times$dist$(s',x_j)$  \\
   \> \> $\forall x_j \in P_i \cup S_i:$  Pos$(x_j)   \leftarrow$ Pos$(x_j) \cup \{ s \}$ \\
  {\bf Return} Pos
 \end{coden}
 \end{scriptsize}
 \end{algorithm}

The input of  \textsc{SamplingPoints} (Algorithm~\ref{alg:sample}) are \texttt{nbruns}, $k$ and $E$.
Here \texttt{nbruns} denotes the number of times the overlapping clusters should be computed, $k$ denotes
the number of clusters 
and $E$ denotes the set of exchange sites. 
\texttt{Pos}$(x_j)$ denotes a set of  metro node positions  of  an exchange site $x_j$.
Initially, \texttt{Pos}$(x_j)$ is an empty set for each exchange site $x_j$.
First,  \textsc{computeOverlappingClusters} is invoked which 
 returns a set of overlapping clusters
 such that each exchange site is present in exactly two clusters. 
Recall  that  a cluster (of exchange sites) $i$ is denoted by $P_i \cup S_i$.
After that an element $s$ is selected from each cluster  
 as a possible metro node position for all the exchange sites within $P_i \cup S_i$.
Also recall that $P_i \cup S_i$ means that the selected  metro node $s$ is the cheapest/closest  for each $e \in P_i$  
and it is second cheapest for  each $e \in S_i$.  
Therefore, if $P_i \neq \emptyset$ then $s$ is selected from $P_i$ 
 such that  the sum of the weighted distances between $s$ and all the exchange sites 
of the cluster  $i$ 
is minimum. 
Otherwise it is selected from $S_i$.
This entire procedure is repeated \texttt{nbruns}  times.
Algorithm~\ref{alg:cluster} can be seen as a variant of weighted k-means clustering algorithm. The main difference 
is that in the original algorithm  clusters are  pairwise mutually
exclusive but Algorithm~\ref{alg:cluster} computes overlapping clusters as required by the problem.

\section{Empirical Results}
\label{sec:results}
In this section we investigate  
different approaches for solving the problem of determining locations of metro nodes in Ireland.   


We used {\scshape cplex} for solving all the integer linear programming formulation of the instances of the double coverage problem.
All of our  algorithms were implemented in  Java.
In our experiments, we varied the number of metro nodes between $18$ and $24$ for Ireland. 
The results are reported  for $19$, $20$, $23$ and $24$ metro nodes.  
The original problem had $1100$ exchange sites. 
In order to do  systematic experimentation, we generated $10$ instances of  smaller sizes.
These instances  are representative of the original instance since they were generated by applying k-means algorithm 
 on the original instance
 by varying $k$ (or  the number of required exchange sites)  from $100$ to $1000$ in steps of $100$.
All the experiments were run  on Linux 2.6.25 x64 on a Dual Quad Core Xeon CPU with overall 11.76 GB of RAM and processor speed of 2.66GHz.

 \begin{table}[t]
  \begin{minipage}[b]{0.5\linewidth}
 \centering%
\begin{scriptsize}
\caption{Results  for  19 metro nodes.}
 \label{tab19}
 \begin{tabular}{|r|c|r|r|r|}
 \hline
    & &  & \multicolumn{2}{c|}{Time (in seconds)}   \\
 $|E|$ & optimal & CBS (GAP) &  MIP & CBS  \\
 \hline
 100 & 470,439,821 &0    \%& 0.25     &0.71  \\
 200 & 475,779,040 &0    \%& 2.29     &1.59  \\
300  & 476,876,335 &0.000\%& 36.90    &3.15  \\
 400 & 477,736,761 &0.009\%& 52.76    &4.40  \\
 500 & 476,930,454 &0.014\%& 96.89    &6.59  \\
 600 & 476,860,839 &0.013\%& 168.47   &8.49  \\
 700 & 477,825,864 &0.012\%& 1,277.27  &14.68 \\
 800 & 477,432,981 &0.033\%& 498.29   &17.43 \\
 900 & 477,608,042 &0.019\%& 817.24   &20.09 \\
1000 & 477,730,261 &0.029\%& 1,081.61  &32.78 \\
1100 & 477,789,473 &0.038\%& 1,716.27  &39.32 \\
 \hline
 \end{tabular}
 \end{scriptsize}
\end{minipage}
\hspace{0.5cm}
 \begin{minipage}[b]{0.5\linewidth}
 \centering
\begin{scriptsize}
\caption{Results for  20 metro nodes.}
 \label{tab20}
 \begin{tabular}{|r|c|r|r|r|}
 \hline
    & &  & \multicolumn{2}{c|}{Time (in seconds)}   \\
 $|E|$ & optimal & CBS (GAP) & MIP & CBS \\
 \hline
 100 & 456,703,030 &0    \%&0.27    &0.69  \\
 200 & 462,745,384 &0    \%&2.26    &1.68  \\
 300 & 463,322,390 &0.001\%&14.84   &2.88  \\
 400 & 464,669,197 &0    \%&70.46   &4.57  \\
 500 & 464,018,395 &0.001\%&115.40  &6.78  \\
 600 & 464,181,132 &0.034\%&226.34  &11.00 \\
 700 & 464,696,666 &0.006\%&405.57  &11.49 \\
 800 & 464,576,759 &0.034\%&661.01  &19.25 \\
 900 & 464,918,687 &0.039\%&1,108.49 &27.26 \\
1000 & 464,968,787 &0.028\%&1,587.39 &31.74 \\
1100 & 465,066,168 &0.034\%&8,777.75 &54.53 \\
 \hline
 \end{tabular}
\end{scriptsize} 
\end{minipage}

 \end{table}
\begin{table}[pth]
\begin{minipage}[b]{0.5\linewidth}
 \centering
 \scriptsize
\caption{Results 
 for  23 metro nodes.}
 \label{tab23}
 \begin{tabular}{|r|c|r|r|r|}
 \hline
    & &  &  \multicolumn{2}{c|}{Time (in seconds)}   \\
 $|E|$ & optimal & CBS (GAP)  & MIP & CBS  \\
 \hline
 100 & 421,504,120 & 0    \%&  0.24    &0.72  \\
 200 & 429,159,208 & 0.048\%&  7.63    &2.19  \\
 300 & 429,880,291 & 0    \%&  19.82   &2.85  \\
 400 & 430,115,650 & 0.005\%&  161.45  &5.41  \\
 500 & 430,043,176 & 0.001\%&  350.34  &9.18  \\
 600 & 429,866,927 & 0.033\%&  713.30  &10.52 \\
 700 & 430,802,977 & 0.019\%&  1,761.50 &18.17 \\
 800 & 430,755,591 & 0.011\%&  2,631.64 &21.79 \\
 900 & 430,737,706 & 0.024\%&  3,858.39 &30.84 \\
1000 & 430,918,149 & 0.008\%&  7,537.79 &39.18 \\
1100 & 430,839,593 & 0.026\%&  9,706.40 &36.61 \\
 \hline
 \end{tabular}
 
\end{minipage}
\hspace{0.5cm}
 \begin{minipage}[b]{0.5\linewidth}
 \centering
 \scriptsize
\caption{Results for 24 metro nodes.}
 \label{tab24}
 \begin{tabular}{|r|c|r|r|r|}
 \hline
    & &  & \multicolumn{2}{c|}{Time (in seconds)}   \\
 $|E|$ & optimal  & CBS (GAP)  & MIP & CBS \\
 \hline
 100 & 411,560,864 & 0	\%& 0.34    &0.71  \\
 200 & 419,088,008 & 0	\%& 9.04    &2.15  \\
 300 & 420,069,722 & 0.005\%& 23.64   &2.93  \\
 400 & 419,722,195 & 0	\%& 68.85   &4.88  \\
 500 & 419,700,725 & 0	\%& 182.44  &7.39  \\
 600 & 419,773,717 & 0.039\%& 293.45  &9.48  \\
 700 & 420,102,946 & 0.008\%& 903.16  &11.71 \\
 800 & 420,352,288 & 0.007\%& 1,532.55 &16.61 \\
 900 & 420,235,833 & 0.011\%& 1,752.59 &21.46 \\
1000 & 420,317,577 & 0.009\%& 3,657.55 &25.93 \\
1100 & 420,347,707 & 0.025\%& 4,316.71 &33.29 \\
 \hline
 \end{tabular}

\end{minipage}
 \end{table}

The results for {\scshape MIP} are presented in Tables~\ref{tab19}-\ref{tab24}.
All the experiments for this approach were run to completion.
The optimal values computed using this approach are shown under the column named ``optimal''.
The results in terms of time (in seconds) are also reported. 
In terms of time this was the most expensive approach especially when the number of exchange sites is more than $500$.

%
%

Although the {\scshape KCN} approach may  solve a problem instance quicker than the {\scshape MIP} approach,
one issue  is to determine the right value of $k$. 
A small  $k$    may result in making the problem  inconsistent and a large $k$  may result in spending more time.
Also, despite having a satisfiable problem  when  $k$ is set   to a relatively lower value, it can still result in spending more time
than that required for solving the original problem, when all the positions are considered for all exchange sites. 
This  is illustrated  in Figure~\ref{fig:kn500} by plotting  the results 
for solving an instance of the double coverage problem
where the number of exchange sites is $500$ and the number of metro nodes is $20$. 
\begin{figure}[t]
\centering
\subfigure[\label{fig:kn500time}]
{
\includegraphics[width=0.48\textwidth]{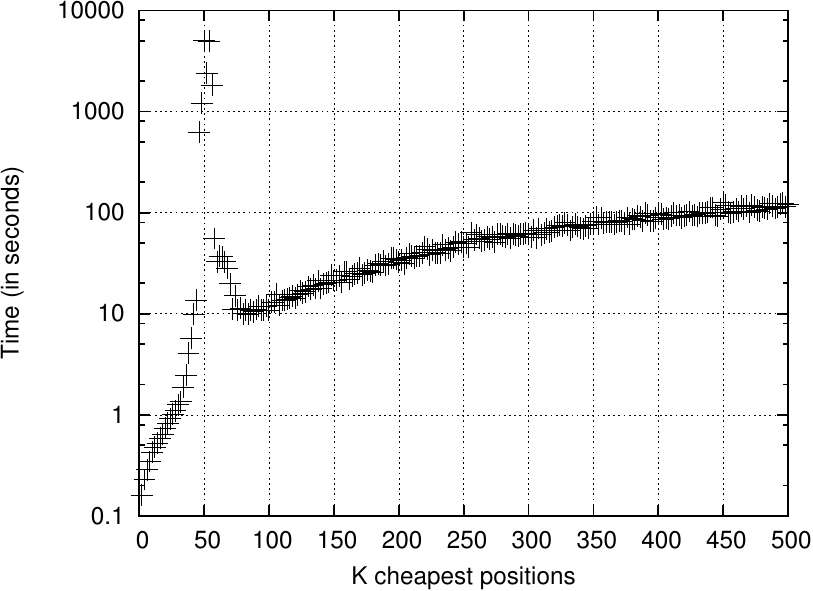}}
\subfigure[ \label{fig:kn500obj}]
{
\includegraphics[width=0.48\textwidth]{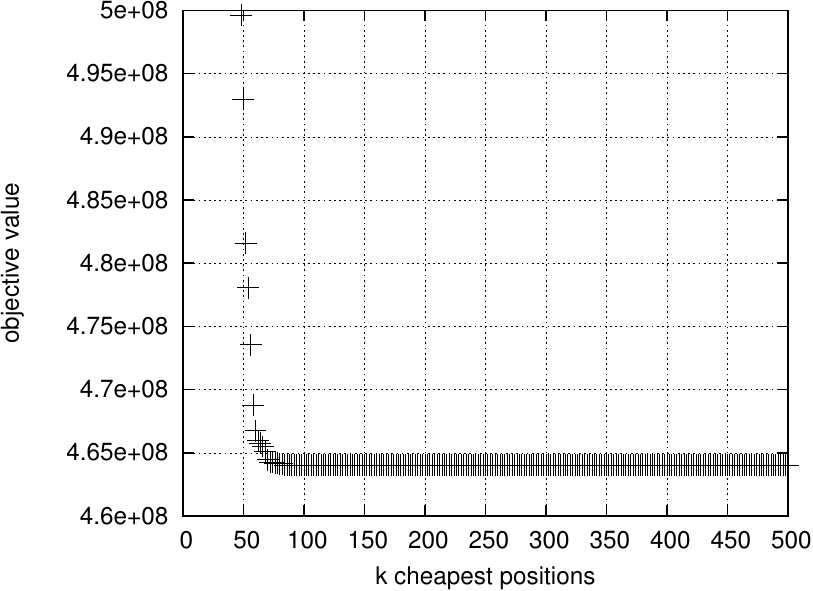}}
\caption{KCN approach  for 20 metro nodes and 500 exchange sites: (a) Time required for different values of $k$. (b)  Optimal value for different  values of $k$ \label{fig:kn500}}
\end{figure}
For both Figures~\ref{fig:kn500time} and \ref{fig:kn500obj} the $x$-axis denotes the value of $k$, which is 
 varied from $2$ to $500$ in steps of $2$.
The $y$-axis of Figure~\ref{fig:kn500time} is the time required to solve the instance and 
the $y$-axis of Figure~\ref{fig:kn500obj} is the optimal value corresponding to $k$. 
Notice that when $k$ is less than or equal to $46$ the problem is always unsatisfiable.
An interesting point to observe is that when $k$ is between $48$ and $56$ 
the time required to solve can be up to $2$ orders-of-magnitude more than that  required  when  $k$ is $500$.   
Also notice that when $k$ is set to $84$ an optimal solution is discovered and the time required to find an optimal solution is also the least. 
The results of the {\scshape KCN} approach are not reported in Tables~\ref{tab19}-\ref{tab24}  for two reasons. 
First, determining the right value of $k$ is not always possible and additionally there is an overhead. 
Second, the other hybrid approach {\scshape CBS} almost always outperforms {\scshape KCN} in terms of time without degrading the quality of the solution.

%

The advantage of the  {\scshape CBS}  approach is that  
if an original instance is satisfiable then a modified instance obtained  by {\scshape CBS}  is also satisfiable. 
Another advantage  is that  
it does not enforce any lower bound  restriction  on the domain size of the metro node positions  for any exchange site. 
An upper bound restriction is implicitly imposed   by the  parameter \texttt{nbruns}
which is equal to the number of times Algorithm~\ref{alg:cluster} is invoked
for computing overlapping clusters. 
The application of  cluster-based sampling  for discarding a set of metro node positions for each exchange site before the search starts
 can be an overhead.
However, it pays off since the time required for search reduces significantly
without sacrificing the quality of the solution as shown in Tables~\ref{tab19}-\ref{tab24}. 
For harder instances it requires almost  two orders-of-magnitude less time than that of the {\scshape MIP} approach.
Also the gap between the cost of the optimal solution and the cost of the best solution found using {\scshape CBS} is within
 $0.05\%$ of the optimal value, which is extremely low.

\section{Conclusions and Future Work}
\label{sec:con}
We have studied and solved the double coverage problem arising in 
long reach passive optical networks that are robust to single node failures.
We showed that the double coverage problem is NP-Complete.
In order to minimize the total length of optical fibre that connects metro nodes and exchange sites 
 we modeled the problem using mixed integer linear programming.
We proposed and studied a hybrid approach that performs 
cluster-based sampling as a preprocessing step in order to reduce the possiblities 
of metro node positions for exchange sites.  We showed 
that
the hybrid approach can reduce the time required to solve the double coverage problem
 by up to two orders-of-magnitude, 
 especially when 
 the size of the problem instance is large.
Our study also shows that the best solutions obtained by using the hybrid approach {\scshape CBS}
  are almost optimal.


The  related work to our contribution in this paper is
 the work on dual-homing protection using MIP~\cite{Wang:04} and local search~\cite{Lee:1997:DMC:261670.261677}.
Although the comparison  with a MIP approach is done, the comparison with a local search approach 
is one of the  future works.
In  future  we would also like to extend our approaches so that they  allow us to specify  the reach of the metro nodes. 
Consequently, this may make some problem instances inconsistent. Therefore it would also be interesting 
 to extend the problem definition where 
 only a given percentage of total customers are required to be dually covered.





\end{document}